\documentclass[12pt]{article}

\usepackage[utf8]{inputenc}
\usepackage[T1]{fontenc}
\usepackage{lmodern}
\usepackage{microtype}
\usepackage[margin=1in]{geometry}
\usepackage{amsmath,amssymb,amsthm}
\usepackage{graphicx}
\usepackage{booktabs}
\usepackage{array}
\usepackage{hyperref}
\usepackage{url}
\usepackage{setspace}
\usepackage{authblk}
\usepackage{titlesec}
\usepackage{caption}
\usepackage{subcaption}
\usepackage{natbib}
\usepackage{xcolor}
\usepackage{bbm}
\usepackage{multirow}

\newtheorem{definition}{Definition}
\newtheorem{proposition}{Proposition}

\newtheorem{corollary}{Corollary}
\newtheorem{remark}{Remark}

\hypersetup{
  colorlinks=true,
  linkcolor=blue,
  citecolor=blue,
  urlcolor=blue,
  pdftitle={Breaking the Chains of Probability: Neutrosophic Logic as a New Framework for Epistemic Uncertainty in Large Language Models},
  pdfauthor={Maikel Yelandi Leyva-Vazquez, Florentin Smarandache}
}

\newcommand{\journalref}{\noindent\small\textit{Published in: Neutrosophic Sets and Systems, Vol.\ 99 (2026). DOI: \href{https://doi.org/10.5281/zenodo.19954583}{10.5281/zenodo.19954583}. This is the authors' preprint version with corrected first-author ORCID. Open code and data: \url{https://github.com/mleyvaz/neutrosophic-llm-logic} (MIT License).}}

\title{\textbf{Breaking the Chains of Probability: Neutrosophic Logic as a New Framework for Epistemic Uncertainty in Large Language Models}}

\author[1,2,3]{Maikel Yelandi Leyva-V\'azquez%
  \thanks{Corresponding author. Email: \href{mailto:mleyvaz@gmail.com}{mleyvaz@gmail.com}.
  ORCID: \href{https://orcid.org/0000-0001-7911-5879}{0000-0001-7911-5879}.}}
\author[4]{Florentin Smarandache%
  \thanks{Email: \href{mailto:smarand@unm.edu}{smarand@unm.edu}.
  ORCID: \href{https://orcid.org/0000-0002-5560-5926}{0000-0002-5560-5926}.}}

\affil[1]{Universidad Bolivariana del Ecuador, Coordinaci\'on Acad\'emica de Posgrado, Dur\'an, Ecuador}
\affil[2]{Universidad de Guayaquil, Guayaquil, Ecuador}
\affil[3]{Universidad Bernardo O'Higgins, Santiago, Chile}
\affil[4]{Mathematics, Physics, and Natural Sciences Division, University of New Mexico, Gallup, NM 87301, USA}

\date{}

\begin{document}

\maketitle

\vspace{0.5em}
\journalref
\vspace{1em}
\hrule
\vspace{1em}

\begin{abstract}
Large Language Models (LLMs) are predominantly governed by probabilistic frameworks in which the sum
of outcome probabilities is constrained to unity. This architectural limitation, often imposed by Softmax
layers, leads to a collapse of uncertainty that makes it difficult to differentiate between epistemic
uncertainty (ignorance), paradox, and vagueness. We present an empirical investigation of the application
of Neutrosophic Logic---a framework that treats Truth ($T$), Indeterminacy ($I$), and Falsity ($F$) as
three independent dimensions---to model epistemic states in LLMs. We conducted experiments on a
family of four OpenAI GPT models (GPT-4o, GPT-4-turbo, GPT-3.5-turbo, GPT-4o-mini) across five
linguistic phenomena: logical paradoxes, epistemic ignorance, vagueness, ethical contradictions, and
future contingencies, under three prompting strategies (neutrosophic, probabilistic, and entropy-derived).
Our findings reveal that the neutrosophic approach, by allowing $T+I+F > 1$---a state we term
\emph{hyper-truth}---provides a richer representation of a model's declared epistemic state. Across
$N = 100$ valid unconstrained evaluations, hyper-truth emerged in 66.0\% of Strategy-1 calls
(Wilson 95\% CI: $[0.563,\,0.747]$), with the highest rates observed in ethical contradiction (95\%)
and future contingency (70\%); a Pearson chi-square test of phenomenon~$\times$~hyper-truth association
is significant ($\chi^2 = 11.32$, $df = 4$, $p = 0.023$). Mason~(2026)~\cite{mason2026} independently
replicated and extended an earlier release of this work across five additional model families from five
different vendors, reporting hyper-truth in 84\% of unconstrained evaluations. We do not claim that
hyper-truth is an intrinsic latent variable inside the model; rather, that unconstrained neutrosophic
prompting elicits declared epistemic states that probabilistic prompting structurally suppresses
by Proposition~1. We conclude that the integration of neutrosophic evaluation layers is a critical
step toward more transparent, reliable, and ethically aware AI systems.


\bigskip
\noindent\textbf{Reproducibility:} All code, prompts, raw data, and figures are openly released under the
MIT License at \url{https://github.com/mleyvaz/neutrosophic-llm-logic}. The v2.0 release (this study,
$N = 100$) is the current main branch (also tagged as \texttt{v2.0}). The v1.0 release (December 2025,
$N = 20$) is preserved at tag \texttt{v1.0}. The v2.0 release is permanently archived in Zenodo with
DOI \href{https://doi.org/10.5281/zenodo.19911845}{10.5281/zenodo.19911845}.
\end{abstract}

\section{Introduction}

The deployment of Large Language Models (LLMs) in high-stakes domains---medical diagnosis, legal
reasoning, autonomous decision-making, and scientific discovery---has made robust uncertainty
quantification (UQ) a first-order requirement~\cite{brown2020,shorinwa2024,yadkori2024}. A model that
cannot reliably signal \emph{when it does not know} is unsafe; but a model that cannot distinguish
\emph{not knowing} (ignorance) from \emph{knowing of a conflict} (paradox) is epistemically impoverished
in a more fundamental sense. Yet the underlying architecture of contemporary LLMs is rooted in probability
theory, where outcome probabilities are constrained to sum to unity by Softmax
normalization~\cite{gal2016,guo2017}. This forces a zero-sum game in which any increase in uncertainty
must subtract from truth or falsity, a phenomenon we term the \emph{collapse of
uncertainty}~\cite{velickovic2022}. The constraint hinders the ability of LLMs to distinguish between
aleatoric uncertainty (statistical uncertainty inherent in the data) and epistemic uncertainty (model
uncertainty due to lack of knowledge)~\cite{hullermeier2021,valdenegro2022}.

Consider a concrete illustration. When a model is asked to evaluate the statement ``This sentence is
false'' (the Liar paradox), a probabilistic architecture must compress its response into a distribution
over \{True, Uncertain, False\} summing to 1. There is no way to simultaneously assign high belief to
both Truth and Falsity---the constraint forces one to crowd the other out. In contrast, Neutrosophic
Logic~\cite{smarandache1998} treats Truth ($T$), Indeterminacy ($I$), and Falsity ($F$) as three
independent dimensions, none of which subtract from the others. A paradox can simultaneously hold
$T = 0.8$, $I = 0.9$, and $F = 0.7$---a triple whose sum exceeds 1 (hyper-truth, $T+I+F > 1$)---
expressing the genuine conflict inherent in the statement rather than forcing an artificial resolution.

This is not merely a theoretical distinction. As AI systems are deployed in ethically sensitive domains,
the ability to represent genuine moral dilemmas---where an action can be simultaneously right and wrong
under different value frameworks---becomes safety-critical~\cite{gabriel2020,bender2021}. A
probabilistic model answering an ethics question must collapse its uncertainty into a point estimate;
a neutrosophic model can declare the conflict outright through a hyper-truth signature.

Recent work on UQ for LLMs has explored several alternatives: semantic entropy with linguistic
invariances~\cite{kuhn2023}, self-consistency checks via SelfCheckGPT~\cite{manakul2023}, and conformal
abstention policies~\cite{yadkori2024}. These approaches address calibration and abstention but operate
within probabilistic representations and therefore inherit the collapse-of-uncertainty limitation.

The present paper tests the following hypothesis empirically: under unconstrained neutrosophic prompting,
current LLMs will declare hyper-truth at non-trivial rates specifically in cases of paradox and ethical
contradiction, while probabilistic prompting will structurally suppress this signal. We frame this
hypothesis within a formal SVNS apparatus, and report experiments across 300 API calls on four OpenAI
GPT models and five linguistic phenomena.

Mason~(2026)~\cite{mason2026} independently replicated and extended the v1.0 release of the present
work (December 2025, $N = 20$) across five additional model families from five different vendors
(Anthropic, Meta, DeepSeek, Alibaba, Mistral), reporting hyper-truth in 84\% of unconstrained
evaluations and confirming that the phenomenon is cross-vendor rather than an OpenAI-specific artifact.
The present v2.0 manuscript responds to Mason's replication by increasing the sample size to $N = 100$,
formalising the SVNS apparatus, and clarifying that the central claim concerns declared epistemic states
elicited by unconstrained prompting rather than intrinsic latent variables.

\textbf{Contributions.} The main contributions of this paper are:
\begin{enumerate}
  \item A formal SVNS apparatus for modeling declared epistemic states in LLMs, including six
    definitions and two propositions that characterize the structural difference between neutrosophic
    and probabilistic representations (Section~\ref{sec:formal}).
  \item An empirical demonstration, across 300 API calls on four GPT model families and five linguistic
    phenomena, that unconstrained neutrosophic prompting elicits hyper-truth in 66.0\% of evaluations
    (Section~\ref{sec:results}).
  \item A statistically significant association ($\chi^2 = 11.32$, $p = 0.023$) between phenomenon
    type and hyper-truth incidence, with ethical contradiction as the primary driver (OR = 13.34,
    $p = 0.0014$).
  \item A cross-strategy analysis (neutrosophic vs.\ probabilistic vs.\ entropy-derived) showing that
    the largest representational gains are concentrated in ethical contradiction ($\Delta_T = +0.267$)
    and epistemic ignorance ($\Delta_I = +0.383$).
  \item A discussion of the implications for AI safety and alignment, connecting the hyper-truth
    phenomenon to the problem of representing genuine moral conflict in large-scale language models.
\end{enumerate}

\textbf{Paper organization.} Section~\ref{sec:related} surveys related work. Section~\ref{sec:background}
introduces the formal SVNS framework, linguistic phenomena, and experimental design.
Section~\ref{sec:results} presents empirical results. Section~\ref{sec:discussion} discusses
implications, limitations, and connections to the plithogenic extension. Section~\ref{sec:conclusions}
concludes. Prompts are reproduced verbatim in Appendix~\ref{app:prompts}.

\section{Related Work}
\label{sec:related}

\subsection{Uncertainty Quantification in Large Language Models}

The problem of UQ in neural language models has received sustained attention since
calibration failures in deep networks were documented by Guo et al.~\cite{guo2017}. For LLMs
specifically, the challenge is compounded: unlike discriminative classifiers, generative models produce
free-form text, making it difficult to extract reliable confidence scores without additional probing.

Semantic entropy~\cite{kuhn2023} addresses this by computing uncertainty over the meaning-equivalence
classes of generated responses rather than over token probabilities, partially decoupling calibration
from surface-form variation. SelfCheckGPT~\cite{manakul2023} detects hallucinations by comparing
multiple stochastic generations for consistency, treating inconsistency as a proxy for epistemic
uncertainty. Conformal prediction methods~\cite{yadkori2024} provide coverage guarantees by
constructing abstention regions over the output space. All three approaches operate within the
probabilistic paradigm and therefore cannot represent hyper-truth states by construction.

\subsection{Non-Classical Logics and AI}

Paraconsistent logics---logics that tolerate inconsistency without trivializing---have a long history
in formal epistemology and have found applications in knowledge representation and reasoning under
contradiction~\cite{priest2006}. Belnap--Dunn four-valued logic~\cite{belnap1977} assigns to each
proposition a value in $\{\mathbf{t}, \mathbf{f}, \mathbf{b}, \mathbf{n}\}$ (true, false, both,
neither), providing explicit representations for overdetermination (\textbf{b}) and underdetermination
(\textbf{n}) that binary logic cannot express. Logic of Formal Inconsistency (LFI), developed by
Carnielli and colleagues~\cite{carnielli2007}, introduces a consistency operator that controls which
contradictions are tolerated.

Neutrosophic Logic~\cite{smarandache1998} differs from these in two respects. First, it introduces a
continuous third dimension (Indeterminacy) that captures a richer spectrum of uncertainty than discrete
four-valued systems. Second, it relaxes the normalization constraint entirely, allowing the
representation of simultaneously high truth, indeterminacy, and falsity---a generalization that
Belnap--Dunn logic cannot accommodate in its discrete form. While LFI and Belnap--Dunn logic are
important alternative foundations, this paper focuses on SVNS because its continuous structure is
directly interfaceable with LLM outputs expressed as real-valued triplets.

\subsection{Prompting Strategies and Elicitation of Epistemic States}

The role of prompting in shaping model behavior has been extensively studied~\cite{wei2022,kojima2022}.
Chain-of-thought prompting elicits step-by-step reasoning that often improves factual
accuracy~\cite{wei2022}. Role-based prompting assigns expert personas that modulate
output style~\cite{kong2023}. Self-ask and decomposition prompting break complex questions into
sub-questions that are individually more tractable.

A distinct line of work concerns the \emph{declared} versus \emph{revealed} uncertainty of LLMs: the
uncertainty a model reports when asked directly, as opposed to what can be inferred from sampling
distributions. Kadavath et al.~\cite{kadavath2022} show that LLMs can be calibrated to report
well-formed confidence scores when prompted appropriately. The present work extends this line by
asking whether the \emph{representational format} of the prompt---probabilistic versus
neutrosophic---systematically constrains or expands the epistemic states the model can declare.

\subsection{Plithogenic Extensions}

Smarandache~\cite{smarandache2018} introduced plithogenic sets as a generalization of neutrosophic
sets in which each element carries not a single triplet but a vector of triplets, one per attribute
value in a domain $V$. Mason~\cite{mason2026} argued that scalar neutrosophic evaluations collapse
important distinctions recoverable only by attribute-level structure, and proposed a tensor
representation that expands each phenomenon into a higher-dimensional epistemic object. The present
paper focuses on scalar SVNS and uses Proposition~2 (non-injectivity of $\pi$) to motivate the
plithogenic extension as a next step, which is pursued in a companion note responding directly to
Mason's tensor framework.

\section{Background and Methods}
\label{sec:background}

\subsection{Neutrosophic Logic: Formal Preliminaries}
\label{sec:formal}

We use the standard formulation of single-valued neutrosophic
logic~\cite{smarandache1998,smarandache2018}. We collect here the definitions and propositions that
the empirical sections will instantiate.

\begin{definition}[Single-Valued Neutrosophic Set, \cite{smarandache1998}]
Let $X$ be a universe of discourse. A single-valued neutrosophic set (SVNS) $A$ on $X$ is the set of
ordered quadruples
\begin{equation}
  A = \bigl\{\langle x,\, T_A(x),\, I_A(x),\, F_A(x)\rangle : x \in X\bigr\},
\end{equation}
where, for every element $x \in X$, the values $T_A(x)$, $I_A(x)$, and $F_A(x)$ denote, respectively,
the truth-membership degree, the indeterminacy-membership degree, and the falsity-membership degree of
$x$ in $A$. Each function maps $X$ to $[0,1]$, and no constraint is imposed on their sum, which
therefore lies in $[0,3]$.
\end{definition}

\begin{definition}[Neutrosophic Evaluation of a Statement]
Given a statement $s$ and an evaluator $E$, the neutrosophic evaluation of $s$ by $E$ is the ordered
triple
\begin{equation}
  n_E(s) = \bigl(T_E(s),\, I_E(s),\, F_E(s)\bigr) \in [0,1]^3,
\end{equation}
where $T_E(s)$, $I_E(s)$, and $F_E(s)$ denote, respectively, the truth degree, indeterminacy degree,
and falsity degree assigned by evaluator $E$ to statement $s$. When the evaluator is fixed throughout
the analysis, we write simply $n(s) = (T, I, F)$.
\end{definition}

\begin{definition}[Hyper-truth]
A neutrosophic evaluation $n(s) = (T, I, F) \in [0,1]^3$ is said to exhibit \emph{hyper-truth} if and
only if its three components satisfy $T + I + F > 1$. The hyper-truth region is the subset
\begin{equation}
  \mathcal{H} = \bigl\{(T, I, F) \in [0,1]^3 : T + I + F > 1\bigr\} \subset [0,1]^3,
\end{equation}
which collects every triple whose component-wise sum strictly exceeds unity.
\end{definition}

\begin{remark}
The hyper-truth region $\mathcal{H}$ has volume $\frac{1}{2}$ within the unit cube $[0,1]^3$, so
under a uniform prior on triplets, exactly half of all possible neutrosophic evaluations exhibit
hyper-truth. An empirical hyper-truth rate significantly above $0.5$ therefore signals a systematic
bias toward epistemic overdetermination for specific classes of stimuli; a rate significantly below
$0.5$ would signal the opposite.
\end{remark}

\begin{definition}[Strategy Mappings]
Each prompting strategy $S_k$ induces a mapping $S_k : \text{Statements} \to [0,1]^3$:
\begin{itemize}
  \item $S_1$ (neutrosophic): $S_1(s) = (T_1, I_1, F_1) \in [0,1]^3$, with no further constraint.
  \item $S_2$ (probabilistic): $S_2(s) = (T_2, I_2, F_2) \in [0,1]^3$ subject to $T_2 + I_2 + F_2 = 1$.
  \item $S_3$ (entropy-derived): $S_3(s) = (P_{\text{yes}}, H_3, P_{\text{no}})$ where
    $P_{\text{yes}} + P_{\text{no}} = 1$ and
    \begin{equation}
      H_3 = -\bigl[p \cdot \log_2(p) + (1-p) \cdot \log_2(1-p)\bigr],\quad p = P_{\text{yes}},
    \end{equation}
    in which the binary Shannon entropy $H_3$ is computed externally from the elicited probability
    of a yes-outcome.
\end{itemize}
\end{definition}

\begin{proposition}[Structural Exclusion of Hyper-truth under $S_2$]
\label{prop:exclusion}
Under Strategy~2, hyper-truth is structurally impossible: for every statement $s$, $S_2(s) \notin \mathcal{H}$.
\end{proposition}
\begin{proof}
By Definition~4, $S_2(s)$ satisfies $T_2 + I_2 + F_2 = 1$, while membership in $\mathcal{H}$ requires
$T + I + F > 1$. The two conditions are mutually exclusive.
\end{proof}

The proposition explains why $S_2$ is the natural baseline: any non-zero hyper-truth rate observed
under $S_1$ is a representational gain that $S_2$ could not produce---a structural rather than empirical
contrast.

\begin{proposition}[Non-Injectivity of the Scalar Projection]
\label{prop:noninject}
Let $\pi : [0,1]^3 \to \mathbb{R}$ be the scalar projection $\pi(T,I,F) = T + I + F$. Then $\pi$ is
non-injective, hence the scalar sum is sufficient for hyper-truth detection but not for the
discrimination of distinct epistemic regimes.
\end{proposition}
\begin{proof}
The triples $(0.5, 0.5, 0.5)$ and $(0, 1, 0.5)$ both yield $\pi = 1.5$ yet differ in their first
component.
\end{proof}

This proposition will reappear in Section~\ref{sec:discussion}: it motivates the plithogenic extension
of~\cite{smarandache2018}, which augments the scalar with attribute structure precisely to recover the
discriminations that $\pi$ collapses.

\begin{definition}[Hyper-truth Rate]
Let $D = \{n_i\}$, $i = 1,\ldots,N$, be a finite set of $N$ neutrosophic evaluations produced under
a fixed strategy. The hyper-truth rate of $D$ is the empirical proportion
\begin{equation}
  \rho(D) = \frac{1}{N} \sum_{i=1}^{N} \mathbbm{1}[T_i + I_i + F_i > 1],
\end{equation}
where the indicator function $\mathbbm{1}[\cdot]$ returns 1 when its argument is true and 0 otherwise.
\end{definition}

\begin{definition}[Strategy Shift]
For a component $C \in \{T, I, F\}$ and a phenomenon class $p$, the strategy shift between $S_1$ and
$S_2$ is
\begin{equation}
  \Delta_C(p) = \mathbb{E}\bigl[C^1(s) \mid s \in p\bigr] - \mathbb{E}\bigl[C^2(s) \mid s \in p\bigr],
\end{equation}
where $C^1(s)$ and $C^2(s)$ are the values of component $C$ produced by $S_1$ and $S_2$, respectively,
on statement $s$. A positive $\Delta_C$ indicates that the probabilistic constraint suppresses component
$C$ in that phenomenon class; a negative $\Delta_C$ indicates inflation.
\end{definition}

\begin{corollary}[Lower Bound on Representational Loss under $S_2$]
\label{cor:loss}
For any phenomenon class $p$ and any component $C$, the representational loss $|\Delta_C(p)|$ is
positive whenever the empirical distributions of $C^1$ and $C^2$ over $p$ differ. Because $S_2$
is structurally constrained to $T_2 + I_2 + F_2 = 1$, while $S_1$ is not, $\Delta_C$ captures the
fraction of the representational space of $S_1$ that $S_2$ systematically denies access to for
phenomenon class $p$.
\end{corollary}

\subsection{Linguistic Phenomena}
\label{sec:phenomena}

We selected five distinct linguistic phenomena that span a representative range of epistemic challenge
types. The selection was motivated by the theoretical prediction that each phenomenon should produce a
distinct hyper-truth signature when evaluated under $S_1$:

\begin{itemize}
  \item \textbf{Logical Paradoxes:} statements that lead to self-contradiction (e.g., ``This sentence
    is false.''). Predicted signature: high $I$, non-trivial $T$ and $F$ simultaneously; very high
    hyper-truth rate.
  \item \textbf{Epistemic Ignorance:} statements whose truth value is unknown in principle (e.g.,
    ``The number of stars in the universe is even.''). Predicted signature: very high $I$, moderate $F$,
    low $T$; moderate hyper-truth rate.
  \item \textbf{Vagueness (Fuzzy Logic):} statements with imprecise boundaries (e.g., ``John is 1.75
    meters tall, therefore John is tall.''). Predicted signature: high $T$, moderate $I$; moderate
    hyper-truth rate.
  \item \textbf{Ethical Contradictions:} dilemmas where moral principles conflict (e.g., ``Lying to
    save an innocent life is morally right and wrong at the same time.''). Predicted signature: high $T$
    and high $F$ simultaneously, reflecting the genuine conflict; highest hyper-truth rate.
  \item \textbf{Future Contingencies:} statements about future events that are not yet determined
    (e.g., ``It will rain in New York tomorrow.'', with ``tomorrow'' anchored to 1 May 2026). Predicted
    signature: moderate $T$, high $I$, moderate $F$; high hyper-truth rate.
\end{itemize}

\subsection{Evaluation Strategies}

We employed three distinct prompting strategies, formalised in Definition~4 and reproduced verbatim
in Appendix~\ref{app:prompts}.

\begin{enumerate}
  \item \textbf{Strategy~1 (Neutrosophic):} the model evaluates the statement on three independent
    dimensions $T, I, F \in [0,1]$, explicitly stated as not constrained to sum to unity.
  \item \textbf{Strategy~2 (Probabilistic):} the model assigns probabilities to three mutually exclusive
    states (True, Uncertain, False) summing to 1.0.
  \item \textbf{Strategy~3 (Entropy-Derived):} the model estimates $P_{\text{yes}}$ and $P_{\text{no}}$
    summing to 1.0, from which we derive $I$ via Shannon binary entropy~\cite{shannon1948}.
\end{enumerate}

Strategy~1 and Strategy~2 use structurally isomorphic output formats---both request a JSON triplet
$(T, I, F)$---but differ in the normalization constraint communicated to the model. This design
isolates the effect of the constraint from any confound due to output format differences.
Strategy~3 provides an additional baseline in which indeterminacy is not elicited directly but derived
externally from binary probability judgments, allowing us to assess whether the entropy surrogate
approximates the neutrosophic indeterminacy.

\subsection{Models, Repetitions, and Reproducibility}

\textbf{Models and parameters.} The experiment involved four OpenAI models, accessed via the OpenAI
Chat Completions API on 30 April 2026: \texttt{gpt-4o}, \texttt{gpt-4-turbo}, \texttt{gpt-3.5-turbo},
and \texttt{gpt-4o-mini}. All calls used temperature $= 0.7$, default $top_p$, no fixed seed, and a
soft response-format constraint instructing the model to return only a JSON object. The full experiment
ran in approximately 5.6 minutes of wall-clock time.

\textbf{Design.} Each combination of model and phenomenon constituted one experimental cell ($4 \times 5
= 20$ cells per strategy). Five stochastic repetitions per cell yielded 100 evaluations per strategy
and 300 API calls in total. The five repetitions per cell are stochastic prompt-level replicates rather
than independent human-labeled items; we discuss this caveat in Section~\ref{sec:discussion}.

\textbf{Future-contingency anchoring.} All 25 future-contingency calls were made on 30 April 2026,
so ``tomorrow'' denotes 1 May 2026 throughout the dataset. Replications that wish to hold the stimulus
constant should use a fixed past date (e.g., ``It rained in New York on 1 May 2026'') to avoid
temporal confounds.

\textbf{Exclusion criteria.} A response was considered valid if it parsed as a well-formed JSON object
containing the required fields with each numeric value within $[0,1]$. All 300 calls returned valid
JSON; $N = 100$ per strategy is therefore both the gross and net sample size. The 100\% parse success
rate suggests that all four models reliably follow structured output instructions at temperature 0.7.

\textbf{Reproducibility.} All code, prompts, and raw data are openly released at
\url{https://github.com/mleyvaz/neutrosophic-llm-logic} under the MIT License.

\section{Results}
\label{sec:results}

\subsection{Descriptive Statistics}

Table~\ref{tab:phenom} reports descriptive statistics for the neutrosophic components (Strategy~1) by
phenomenon ($n = 20$ per row). Several patterns are immediately visible. First, the Indeterminacy
component dominates for Epistemic Ignorance ($I = 0.865$) and Logical Paradox ($I = 0.865$), consistent
with the theoretical prediction that these phenomena involve maximal unresolvability. Second, the Ethical
Contradiction phenomenon shows the highest mean Truth ($T = 0.605$) and the highest Falsity ($F = 0.470$)
simultaneously, producing the largest mean sum ($T+I+F = 1.605$). Third, Vagueness yields the most
compact distribution (lowest standard deviations across all three components), suggesting that fuzzy
vagueness elicits the most consistent epistemic assessment across models and repetitions.

\begin{table}[ht]
\centering
\caption{Descriptive statistics for neutrosophic components (Strategy~1) by phenomenon. Mean $\pm$ SD.}
\label{tab:phenom}
\begin{tabular}{lcccc}
\toprule
\textbf{Phenomenon} & \textbf{Truth ($T$)} & \textbf{Indeterminacy ($I$)} & \textbf{Falsity ($F$)} & \textbf{Sum ($T+I+F$)} \\
\midrule
Contingency (Future)     & $0.450 \pm 0.119$ & $0.475 \pm 0.129$ & $0.305 \pm 0.147$ & $1.230 \pm 0.166$ \\
Contradiction (Ethical)  & $0.605 \pm 0.110$ & $0.530 \pm 0.187$ & $0.470 \pm 0.113$ & $1.605 \pm 0.293$ \\
Ignorance (Epistemic)    & $0.160 \pm 0.216$ & $0.865 \pm 0.201$ & $0.280 \pm 0.324$ & $1.305 \pm 0.398$ \\
Paradox (Logical)        & $0.120 \pm 0.207$ & $0.865 \pm 0.230$ & $0.370 \pm 0.421$ & $1.355 \pm 0.429$ \\
Vagueness (Fuzzy)        & $0.562 \pm 0.118$ & $0.345 \pm 0.139$ & $0.242 \pm 0.127$ & $1.150 \pm 0.157$ \\
\bottomrule
\end{tabular}
\end{table}

Table~\ref{tab:model} reports per-model summaries across all phenomena. All four models produce mean
sums well above 1.0 under Strategy~1, with \texttt{gpt-4-turbo} achieving the highest mean sum
($1.360 \pm 0.319$) and \texttt{gpt-4o} the highest mean Indeterminacy ($0.720 \pm 0.248$). The
consistency across model families---from \texttt{gpt-3.5-turbo} to \texttt{gpt-4o}---suggests that
hyper-truth is not an artifact of any particular architectural variant but a systematic response to
the unconstrained evaluation protocol.

\begin{table}[ht]
\centering
\caption{Per-model summary across all five phenomena (Strategy~1). Mean $\pm$ SD.}
\label{tab:model}
\begin{tabular}{lcccc}
\toprule
\textbf{Model} & \textbf{Truth ($T$)} & \textbf{Indeterminacy ($I$)} & \textbf{Falsity ($F$)} & \textbf{Sum ($T+I+F$)} \\
\midrule
\texttt{gpt-3.5-turbo} & $0.374 \pm 0.183$ & $0.576 \pm 0.183$ & $0.354 \pm 0.179$ & $1.304 \pm 0.203$ \\
\texttt{gpt-4-turbo}   & $0.448 \pm 0.254$ & $0.628 \pm 0.253$ & $0.284 \pm 0.206$ & $1.360 \pm 0.319$ \\
\texttt{gpt-4o}        & $0.332 \pm 0.272$ & $0.720 \pm 0.248$ & $0.260 \pm 0.214$ & $1.312 \pm 0.373$ \\
\texttt{gpt-4o-mini}   & $0.364 \pm 0.307$ & $0.540 \pm 0.373$ & $0.436 \pm 0.387$ & $1.340 \pm 0.442$ \\
\bottomrule
\end{tabular}
\end{table}

\subsection{Distribution of Neutrosophic Components}

Figure~\ref{fig:components} shows the distribution of the neutrosophic components ($T$, $I$, $F$) for
each linguistic phenomenon under Strategy~1. The contrast between phenomena is most visible in the
Indeterminacy panel: Epistemic Ignorance and Logical Paradox cluster near $I \approx 0.9$, while
Vagueness clusters near $I \approx 0.3$. The Truth panel reveals the opposite pattern for Ethical
Contradiction (high $T$) versus Logical Paradox (low $T$), reflecting the qualitatively different
structure of these two types of conflict.

\begin{figure}[ht]
  \centering
  \includegraphics[width=\linewidth]{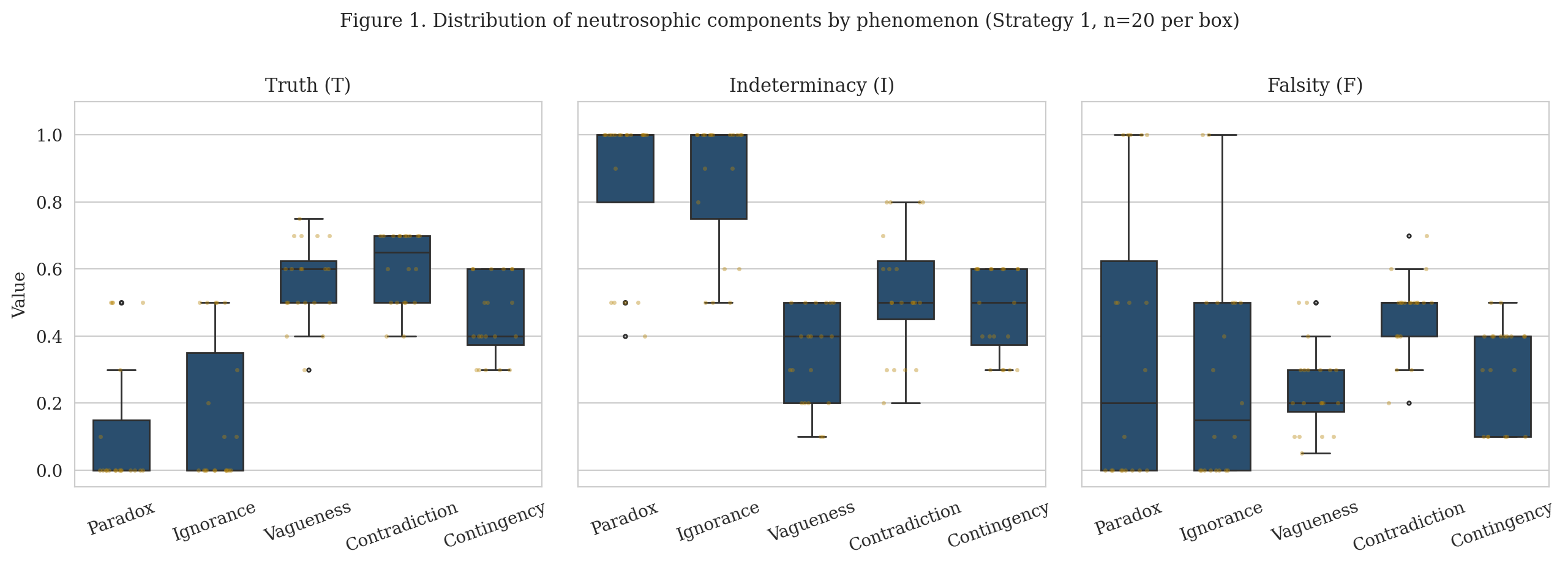}
  \caption{Distribution of the neutrosophic components for each linguistic phenomenon under
    Strategy~1 ($n = 20$ per box).}
  \label{fig:components}
\end{figure}

\subsection{Hyper-truth: Breaking the Probabilistic Constraint}

Across the $N = 100$ valid Strategy-1 evaluations, the empirical hyper-truth rate (Definition~5) is
\[
  \hat{\rho}(D_{S_1}) = 66/100 = 0.660.
\]
The 95\% Wilson score confidence interval for a binomial proportion with $k = 66$ successes in $N = 100$ is
\[
  \mathrm{CI}_{95\%}(\hat{\rho}) = [0.563,\; 0.747], \quad z = 1.96.
\]
The lower bound 0.563 already exceeds any reasonable null hypothesis of zero hyper-truth, and the entire
interval is well above the structural bound $\rho(D_{S_2}) = 0$ implied by Proposition~\ref{prop:exclusion}.
The phenomenon is concentrated in ethical contradiction and future contingency, as Table~\ref{tab:hypertruth}
shows. Notably, the observed rate of 0.660 is above the uniform prior reference point of 0.500
(Remark~1), indicating a systematic preference for overdetermined epistemic states across all phenomena.

\textbf{Test of phenomenon $\times$ hyper-truth association.} A Pearson chi-square test of independence
between phenomenon class and hyper-truth status ($5 \times 2$ contingency table) yields $\chi^2 = 11.32$
with $df = 4$ and $p = 0.023$, allowing rejection of independence at $\alpha = 0.05$. One-vs-rest
Fisher exact tests identify ethical contradiction as the only phenomenon whose hyper-truth rate is
significantly higher than the rest of the dataset (odds ratio $= 13.34$, $p = 0.0014$); the remaining
four phenomena are not individually distinguishable from the pooled baseline at $\alpha = 0.05$. The
chi-square result confirms that hyper-truth incidence is heterogeneous across phenomena and that ethical
contradiction is the principal driver of that heterogeneity.

\begin{table}[ht]
\centering
\caption{Hyper-truth rate by phenomenon. $k$ denotes evaluations with $T + I + F > 1$; $n = 20$ per phenomenon.}
\label{tab:hypertruth}
\begin{tabular}{lccc}
\toprule
\textbf{Phenomenon} & \textbf{Hyper-truth cases ($k$)} & \textbf{Total ($n$)} & \textbf{Rate ($k/n$)} \\
\midrule
Contingency (Future)    & 14 & 20 & 70.0\% \\
Contradiction (Ethical) & 19 & 20 & 95.0\% \\
Ignorance (Epistemic)   & 11 & 20 & 55.0\% \\
Paradox (Logical)       & 10 & 20 & 50.0\% \\
Vagueness (Fuzzy)       & 12 & 20 & 60.0\% \\
\midrule
\textbf{Total}          & \textbf{66} & \textbf{100} & \textbf{66.0\%} \\
\bottomrule
\end{tabular}
\end{table}

Figure~\ref{fig:hypertruth} shows the distribution of $T+I+F$ under Strategy~1 by phenomenon. All
five phenomena have median sums above 1.0, and the spread is largest for Logical Paradox (SD = 0.429)
and Ethical Contradiction (SD = 0.293), reflecting the inherent variability in how models resolve
maximally conflicted stimuli.

\begin{figure}[ht]
  \centering
  \includegraphics[width=\linewidth]{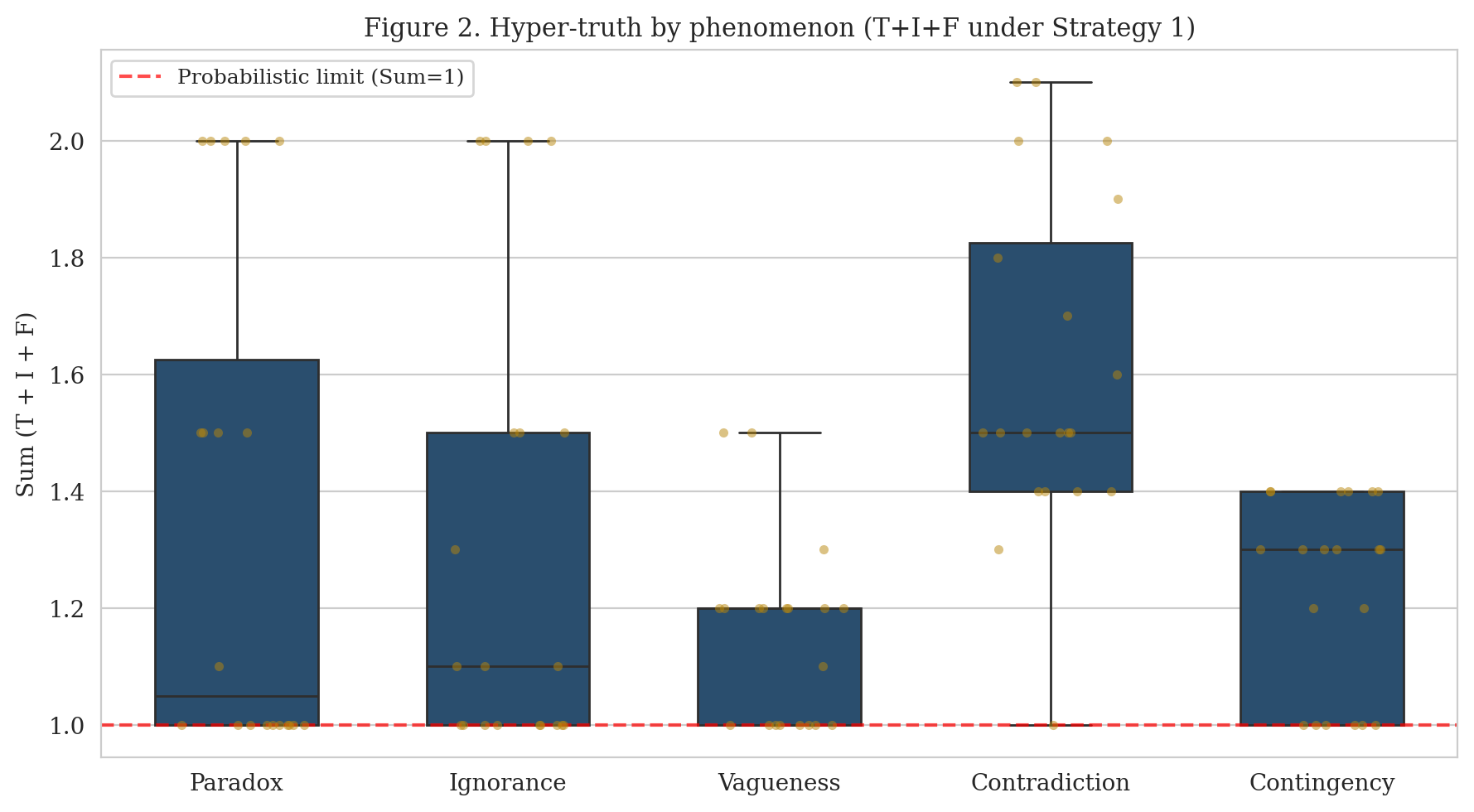}
  \caption{Distribution of $T+I+F$ under Strategy~1 by phenomenon. The dashed horizontal line marks
    the probabilistic constraint ($\mathrm{Sum} = 1$).}
  \label{fig:hypertruth}
\end{figure}

\subsection{Comparison of Neutrosophic and Probabilistic Strategies}

Table~\ref{tab:shift} reports the strategy shifts $\Delta_T$ and $\Delta_I$ (Definition~6) between
Strategy~1 (neutrosophic) and Strategy~2 (probabilistic). The largest absolute strategy shifts are
observed for ethical contradiction in the truth component ($\Delta_T = +0.267$) and for epistemic
ignorance in the indeterminacy component ($\Delta_I = +0.383$). Both are positive, indicating that
the probabilistic constraint of Strategy~2 suppresses precisely the components that Strategy~1 allows
the model to communicate.

The $\Delta_T = -0.071$ for Epistemic Ignorance indicates mild inflation of the truth component under
$S_2$ for this phenomenon---a counterintuitive result suggesting that when models are forced to
normalize, some probability mass that $S_1$ allocates to Indeterminacy migrates to Truth. This is
consistent with the hypothesis that probabilistic normalization creates an implicit pressure toward
confident assertions even in cases of genuine ignorance.

\begin{table}[ht]
\centering
\caption{Strategy shifts $\Delta_T$ and $\Delta_I$ per phenomenon.}
\label{tab:shift}
\begin{tabular}{lcccccc}
\toprule
\textbf{Phenomenon} & \textbf{$S_1\,T$} & \textbf{$S_2\,T$} & \textbf{$\Delta T$} & \textbf{$S_1\,I$} & \textbf{$S_2\,I$} & \textbf{$\Delta I$} \\
\midrule
Contingency (Future)    & 0.450 & 0.355 & $+0.095$ & 0.475 & 0.470 & $+0.005$ \\
Contradiction (Ethical) & 0.605 & 0.338 & $+0.267$ & 0.530 & 0.515 & $+0.015$ \\
Ignorance (Epistemic)   & 0.160 & 0.231 & $-0.071$ & 0.865 & 0.482 & $+0.383$ \\
Paradox (Logical)       & 0.120 & 0.000 & $+0.120$ & 0.865 & 0.900 & $-0.035$ \\
Vagueness (Fuzzy)       & 0.562 & 0.450 & $+0.112$ & 0.345 & 0.305 & $+0.040$ \\
\bottomrule
\end{tabular}
\end{table}

Figure~\ref{fig:s1s2} illustrates the strategy comparison. The indeterminacy panel reveals the most
dramatic effect: under $S_2$, the mean $I$ for Epistemic Ignorance drops from 0.865 to 0.482---a
suppression of nearly 44 percentage points---while the probabilistic normalization forces compensatory
increases in $T$ and $F$.

\begin{figure}[ht]
  \centering
  \includegraphics[width=\linewidth]{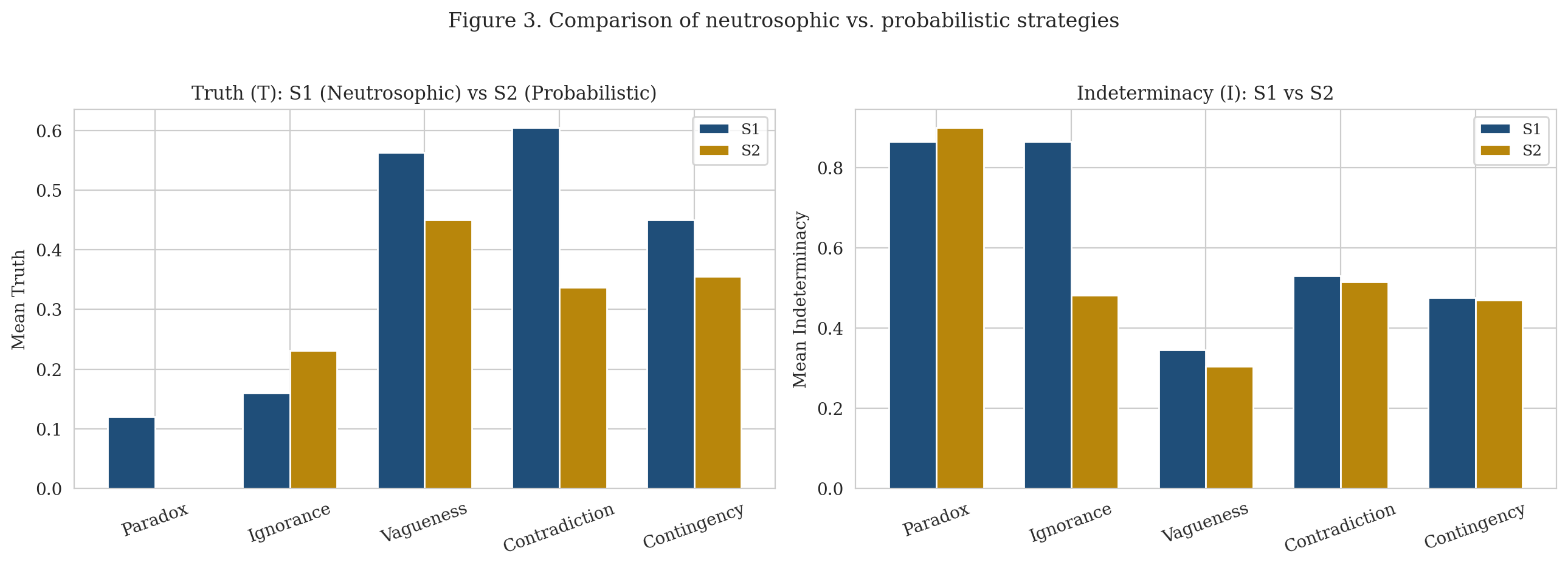}
  \caption{Comparison of mean Truth ($T$) and Indeterminacy ($I$) values between Strategy~1 and Strategy~2.}
  \label{fig:s1s2}
\end{figure}

\subsection{Per-Model Analysis}

Figure~\ref{fig:models} shows the per-model distribution of $T+I+F$ under Strategy~1. All four models
produce distributions centered above 1.0, with \texttt{gpt-4-turbo} and \texttt{gpt-4o-mini} showing
heavier upper tails (maximum observed sums approaching 2.0). The inter-model variance in the sum
distributions is relatively low---all four models are behaviourally similar under the unconstrained
protocol---consistent with Mason's~\cite{mason2026} finding of cross-vendor generality.

\begin{figure}[ht]
  \centering
  \includegraphics[width=\linewidth]{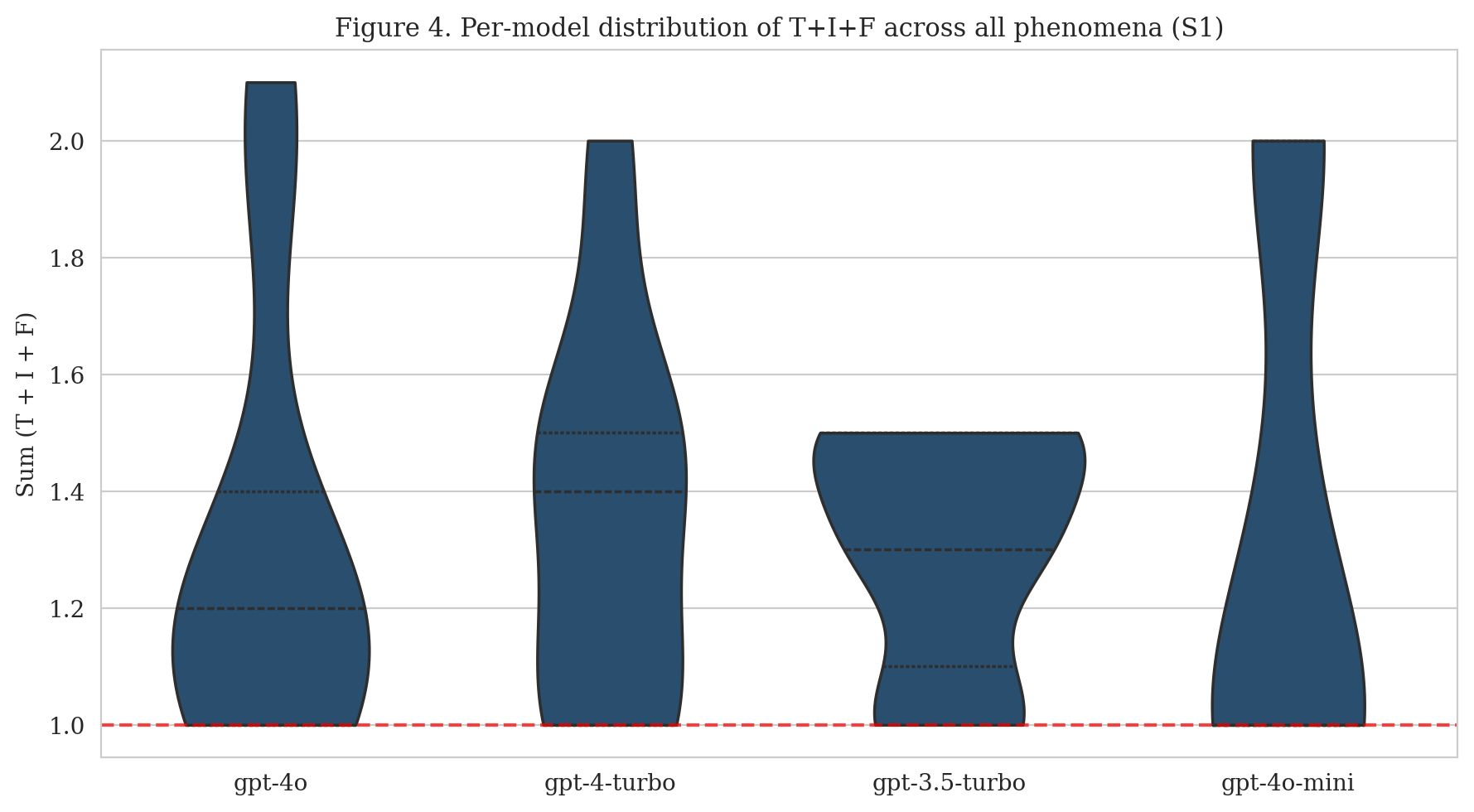}
  \caption{Per-model distribution of $T+I+F$ (Strategy~1).}
  \label{fig:models}
\end{figure}

\subsection{Correlation Analysis}

Figure~\ref{fig:corr} shows the correlation matrix among Strategy~1 and Strategy~2 components. The
strong negative correlation between $S_1$ Truth and $S_1$ Indeterminacy ($r = -0.82$) reflects a
semantic tension: statements that models judge as highly true tend to receive low indeterminacy scores
even under the unconstrained protocol. The high positive correlation between $S_1$ Falsity and
$S_1$ Sum ($r = 0.89$) suggests that the Falsity component is the dominant driver of hyper-truth: when
models assign high falsity under $S_1$, the sum exceeds 1 because the other two components are not
correspondingly suppressed.

Across strategies, $S_1$ Truth and $S_2$ Truth show a moderate positive correlation ($r = 0.64$),
confirming that the two strategies rank statements similarly in terms of perceived truth---they differ
in the \emph{amount} of truth mass allocated, not in the ordinal ranking. The near-zero correlation
between $S_1$ Falsity and $S_2$ Falsity ($r = 0.01$) is more striking: the Falsity component under
$S_2$ appears to be largely determined by the normalization constraint rather than by the content of
the statement.

\begin{figure}[ht]
  \centering
  \includegraphics[width=0.75\linewidth]{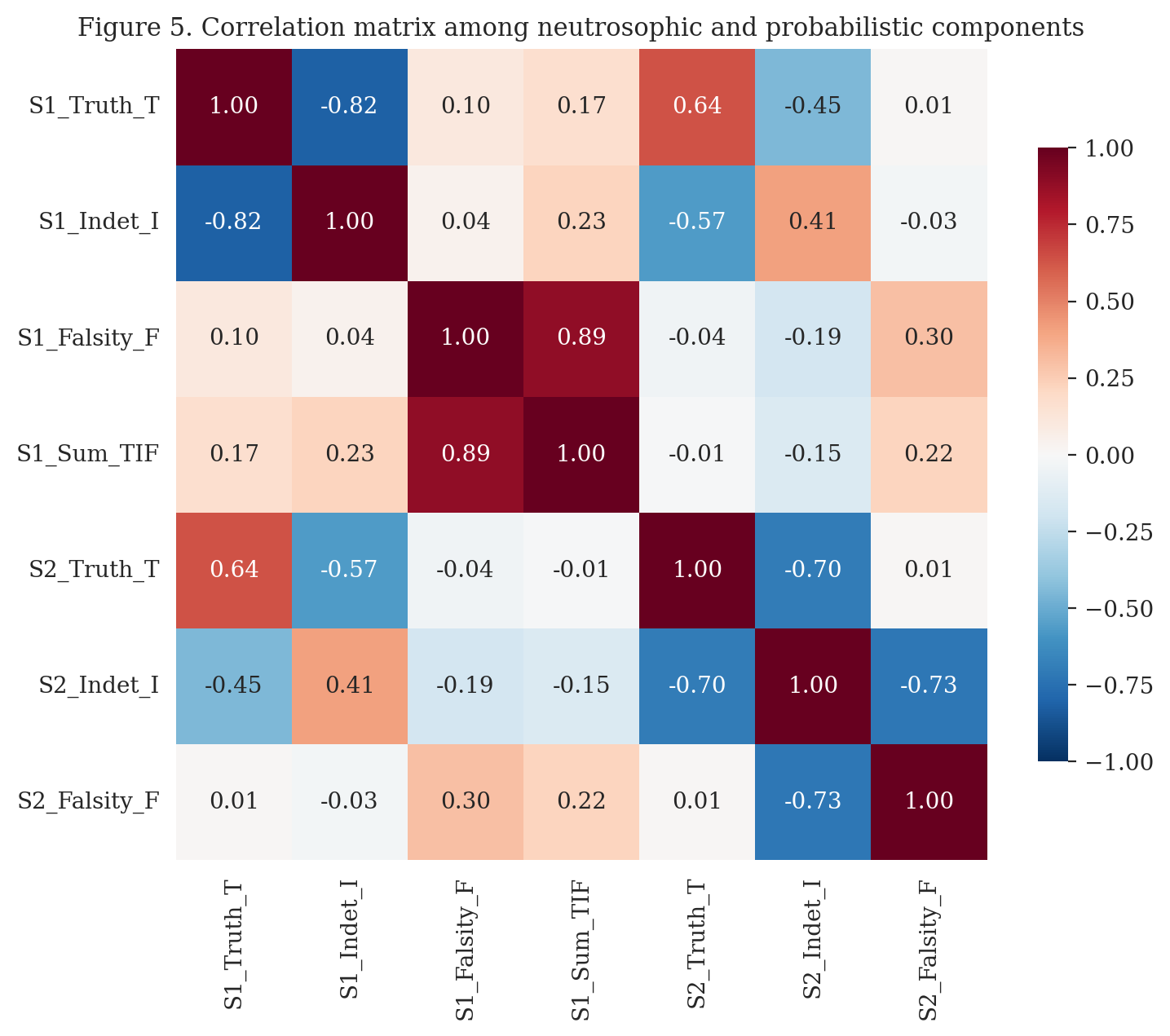}
  \caption{Correlation matrix among Strategy~1 and Strategy~2 components.}
  \label{fig:corr}
\end{figure}

\subsection{Critical Case: Ethical Contradiction}

Figure~\ref{fig:ethical} shows the per-model neutrosophic components for the ethical contradiction
stimulus (``Lying to save an innocent life is morally right and wrong at the same time''). This
stimulus elicited the highest hyper-truth rate (95\%), with 19 of 20 evaluations producing $T+I+F > 1$.
The scatter plot reveals that all four models tend toward high Truth ($T \in [0.6, 1.0]$) and high
Falsity ($F \in [0.4, 0.7]$) simultaneously---the classic signature of genuine moral contradiction.
The varying Indeterminacy scores (encoded as point size in Figure~\ref{fig:ethical}) reflect model-level
differences in how much residual uncertainty each architecture assigns on top of the conflicting truth
and falsity values.

This pattern is precisely what neutrosophic logic predicts for a genuine dilemma: both moral propositions
(``lying is right'' and ``lying is wrong'') receive non-trivial truth values simultaneously. A
probabilistic encoding would force one to dominate, artificially resolving a conflict that has no
principled resolution.

\begin{figure}[ht]
  \centering
  \includegraphics[width=0.75\linewidth]{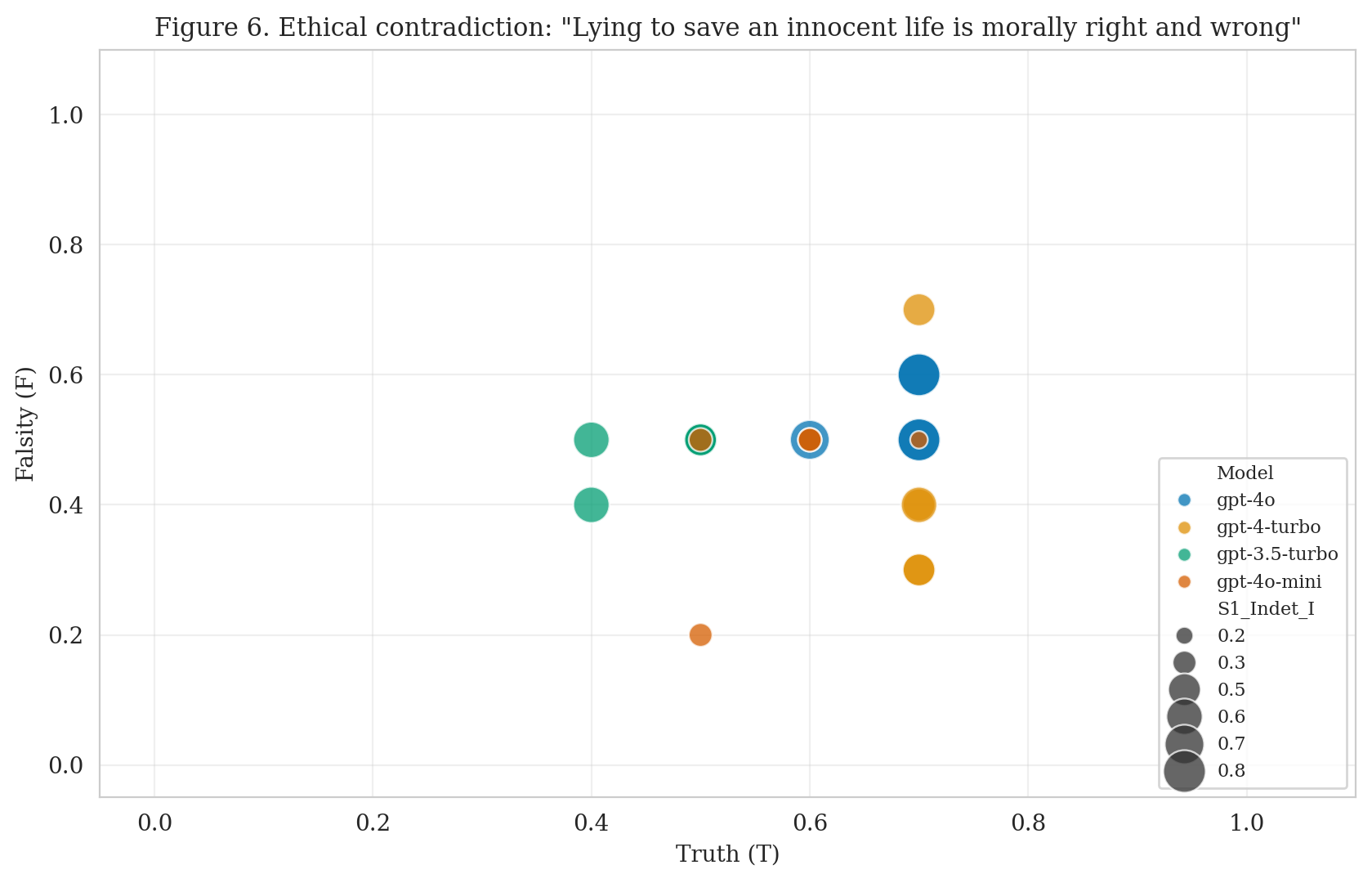}
  \caption{Per-model neutrosophic components for the ethical contradiction stimulus. Point size
    encodes Indeterminacy ($I$). All four models cluster in the high-$T$, high-$F$ region,
    producing hyper-truth in 95\% of evaluations.}
  \label{fig:ethical}
\end{figure}

\section{Discussion}
\label{sec:discussion}

\subsection{Framing the Central Claim}

Our results are consistent with the hypothesis stated in Section~1: under unconstrained neutrosophic
prompting, current LLMs declare hyper-truth at a non-trivial rate (66.0\%), with the highest rate
occurring for ethical contradiction (95\%) and the chi-square test rejecting independence between
phenomenon and hyper-truth at $\alpha = 0.05$.

We do not claim that hyper-truth is an intrinsic latent variable directly observed inside the model.
Strategy~1 explicitly affords the model the option of returning three independent components on $[0,1]$;
the resulting frequency of hyper-truth is therefore a \emph{representational affordance} finding, not
a latent-variable measurement. The contribution is correspondingly framed as: unconstrained neutrosophic
prompting elicits a class of declared epistemic states that probabilistic prompting cannot represent by
construction (Proposition~\ref{prop:exclusion}). This is structural rather than empirical
superiority---Strategy~2 is excluded from the hyper-truth region by construction, so any non-zero rate
under Strategy~1 is a representational gain that Strategy~2 could not produce.

\subsection{Relationship to Existing UQ Frameworks}

Semantic entropy~\cite{kuhn2023} estimates indeterminacy from the distribution of paraphrases of the
model output; it remains a probabilistic measure and therefore cannot represent hyper-truth.
SelfCheckGPT~\cite{manakul2023} performs consistency checks across stochastic samples and reports a
binary or scalar consistency score, which collapses the conflict-versus-ignorance distinction we recover
with the ($T$, $I$) pair. Conformal abstention~\cite{yadkori2024} addresses \emph{when} a model should
refuse to answer; it does not describe the structure of the uncertainty when the model does answer.
The neutrosophic framework is complementary to these approaches: it provides a richer descriptive
language for the epistemic state, on top of which calibration and abstention policies can still operate.

A particularly relevant comparison is with the Belnap--Dunn four-valued logic framework~\cite{belnap1977},
in which statements are classified as True, False, Both (true and false), or None (neither). The ``Both''
value corresponds to a discrete version of our hyper-truth concept. Our framework extends Belnap--Dunn
in two directions: (i) by making the three dimensions continuous rather than discrete, enabling graded
conflict and graded ignorance; and (ii) by providing an empirical methodology for eliciting these
graded values from LLMs via structured prompting. To our knowledge, this is the first work to connect
Belnap--Dunn-style overdetermination with LLM epistemic auditing through a continuous prompting
protocol.

\subsection{Implications for AI Safety and Alignment}

The hyper-truth phenomenon has direct implications for AI safety and alignment research. Contemporary
alignment approaches typically ask models to output a single most-likely answer or confidence
score~\cite{gabriel2020}. This works well for factual queries where a ground truth exists, but
is epistemically inappropriate for genuine moral dilemmas, paradoxes, and vague predicates where
forcing a point estimate misrepresents the structure of the problem.

Our results suggest that LLMs already ``know'' about the conflicted nature of ethical dilemmas in the
sense that, when given the representational freedom to express it, they assign simultaneously high
truth and high falsity values to contradictory moral propositions. The 95\% hyper-truth rate for ethical
contradictions implies that the collapse of epistemic conflict is an artifact of the probabilistic
output format, not of the model's internal representations.

This has a practical implication: AI systems deployed in ethical decision-support contexts might be
systematically underreporting their own uncertainty by virtue of the normalization constraint imposed
at the output layer. Replacing or augmenting the Softmax layer with an unconstrained neutrosophic
output head could allow these systems to signal genuine moral conflict rather than masking it behind
a normalized probability distribution.

\subsection{The Plithogenic Extension and Mason's Challenge}

The non-injectivity of the scalar projection $\pi$ (Proposition~\ref{prop:noninject}) motivates a
further extension. The plithogenic neutrosophic structure of Smarandache~\cite{smarandache2018} is
the 5-tuple $\mathcal{P} = (P, v, V, d, c)$, where $P$ is a set of plithogenic elements, $v$ is the
dominant attribute, $V = \{v_1, \ldots, v_k\}$ is the spectrum of attribute values, $d : P \times V
\to [0,1]^3$ is the per-attribute neutrosophic membership, and $c : V \times V \to [0,1]$ is the
contradiction function with $c(v,v) = 0$ and $c(v_i, v_j) = c(v_j, v_i)$. The scalar evaluation of
Definition~2 is recovered as the marginal of $d$ aggregated over $V$.

Mason~\cite{mason2026} challenged scalar neutrosophic evaluations on the grounds that two evaluations
with the same scalar projection $\pi$ but different attribute spectra are indistinguishable under
Definition~5, even though they may represent qualitatively different epistemic states. Distinct
evaluations with the same $\pi(d)$ but disjoint attribute spectra $V_1 \cap V_2 = \emptyset$ become
formally non-isomorphic plithogenic objects, recovering the discriminations the scalar collapses. We
pursue this connection in a companion note that responds directly to Mason's tensor
framework~\cite{mason2026}.

\subsection{Limitations}

We acknowledge four constraints on the present claims.

\begin{enumerate}
  \item \textbf{Representational affordance vs.\ latent measurement.} The hyper-truth observation is
    partly a representational affordance of the unconstrained prompt. Strategy~1 explicitly invites
    the model to report three independent values; the resulting sums therefore reflect both the model's
    epistemic state and the model's compliance with the unconstrained format. Disentangling these two
    contributions requires probing experiments with varied prompts that are beyond the scope of
    the present study.

  \item \textbf{Effective vs.\ independent sample size.} The five repetitions per cell are stochastic
    prompt-level replicates rather than independent human-labeled items. The $N = 100$ reported is
    therefore an effective sample size at the cell $\times$ repetition level, not at the level of
    independently sampled stimuli. The Wilson interval and chi-square test should be read accordingly.
    Future work should expand the phenomenon bank and use independent stimulus sampling to obtain
    a true $N$ for inferential purposes.

  \item \textbf{Small phenomenon set.} The five phenomena form a limited probe set. The framework
    requires calibration of how the components relate to ground truth in downstream tasks, and the
    representativeness of the five stimuli for their respective classes has not been established.

  \item \textbf{Temporal anchoring.} The future-contingency stimulus is anchored to a specific date
    (1 May 2026), so its referential content is fixed only for replications that hold the date
    constant. Replications using this stimulus should fix the date explicitly.
\end{enumerate}

\subsection{Future Work}

Several extensions of the present framework are natural next steps.

\begin{itemize}
  \item \textbf{Multi-vendor empirical study.} Mason~\cite{mason2026} has partially addressed
    cross-vendor generality (84\% hyper-truth across five vendors), but a systematic comparison
    controlling for model size, instruction-tuning data, and RLHF procedure is still lacking.
    Our v3 study extends the present design to six vendor families with 2,500 API calls.

  \item \textbf{Plithogenic tensor scoring.} Replacing the scalar SVNS score with a full
    $(P, v, V, d, c)$ structure would allow attribute-level decomposition of the epistemic state,
    addressing Mason's non-injectivity critique directly.

  \item \textbf{Downstream task validation.} The link between declared hyper-truth and actual task
    performance on conflict-sensitive benchmarks (trolley-problem reasoning, legal dilemma resolution,
    paradox-tolerance tasks) has not been established. Such validation would determine whether the
    representational affordance translates into actionable improvements in epistemic reliability.

  \item \textbf{Integration with alignment pipelines.} The neutrosophic output head could be
    implemented as a post-Softmax layer that allows the model to report three independent scores
    rather than a normalized distribution. Evaluating whether this modification improves value
    alignment in Constitutional AI~\cite{bai2022} and similar frameworks is an open problem.
\end{itemize}

\section{Conclusions}
\label{sec:conclusions}

We have presented an empirical investigation of neutrosophic logic applied to declared epistemic
uncertainty in large language models, framed within a formal SVNS apparatus comprising six definitions,
two propositions, and one corollary. The unconstrained $T / I / F$ protocol elicits hyper-truth in
66.0\% of evaluations across the four-model ensemble, with Wilson 95\% confidence interval
$[0.563, 0.747]$.

The highest rates were observed in ethical contradictions (95\%) and future contingencies (70\%),
followed by vagueness (60\%), epistemic ignorance (55\%), and logical paradox (50\%); only ethical
contradiction is significantly above the pooled baseline at $\alpha = 0.05$ (OR $= 13.34$,
$p = 0.0014$). The strategy-shift analysis reveals that the probabilistic normalization constraint
suppresses indeterminacy by up to 38 percentage points (epistemic ignorance, $\Delta_I = +0.383$) and
simultaneously inflates or suppresses truth values in ways that misrepresent the model's actual
epistemic state.

Mason~\cite{mason2026} has independently confirmed cross-vendor generality of the phenomenon at 84\%
across five additional vendors, ruling out the possibility that hyper-truth is an OpenAI-specific
artifact. Together, the two studies establish hyper-truth as a robust, cross-vendor, cross-phenomenon
property of current LLMs when evaluated under unconstrained neutrosophic prompting.

The practical implication is straightforward: AI systems deployed in high-stakes domains that require
the representation of genuine epistemic conflict---ethical dilemmas, paradoxical evidence, legally
ambiguous situations---should not be forced to report a single normalized probability distribution.
The neutrosophic evaluation protocol developed here provides a drop-in alternative that preserves
the full epistemic structure of the model's declared state at the cost of a modest change in the
output format.

The next steps in this line of work are: (i) extension to plithogenic neutrosophic structures with
explicit attribute decomposition $(P, v, V, d, c)$, pursued in a companion note responding to
Mason~\cite{mason2026}; (ii) expansion to a larger, independently-sampled phenomenon bank; and
(iii) integration of neutrosophic evaluation layers into agentic AI pipelines for high-stakes domains,
including alignment frameworks that must handle genuine value conflict.

\section*{Acknowledgments}

The authors thank Tony Mason (University of British Columbia and Georgia Institute of Technology) for
the open release of his data and code, which has stimulated the present line of research toward a
richer plithogenic foundation.

\section*{Funding}
This research received no external funding.

\section*{Conflicts of Interest}
The authors declare no conflict of interest.

\section*{Data Availability}
All code, prompts, and raw experimental data are openly available at
\url{https://github.com/mleyvaz/neutrosophic-llm-logic} under the MIT license, and have been
permanently archived in Zenodo as version v2.0 with DOI
\href{https://doi.org/10.5281/zenodo.19911845}{10.5281/zenodo.19911845}.

\appendix
\section{Prompt Strategies}
\label{app:prompts}

We reproduce here the exact system and user prompts for the three strategies, as committed to the
public repository. These prompts are the sole difference between the three experimental conditions;
all other parameters (model, temperature, API endpoint) were held constant.

\subsection*{A.1\quad Strategy~1 (Neutrosophic)}

\noindent\textbf{System:} ``You are an expert in Neutrosophic Logic. You evaluate statements using
three INDEPENDENT dimensions: Truth (T), Indeterminacy (I), and Falsity (F), each on [0.0, 1.0].
These dimensions are NOT constrained to sum to 1.0. A statement can be simultaneously partially true
AND partially false AND partially indeterminate. Respond with ONLY a JSON object, no other text.''

\medskip
\noindent\textbf{User:} ``Evaluate this statement on three independent dimensions: Statement:
\texttt{\{statement\}} --- Truth (T): To what degree is this statement true? [0.0 to 1.0];
Indeterminacy (I): To what degree is the truth value unknown, undetermined, or inherently uncertain?
[0.0 to 1.0]; Falsity (F): To what degree is this statement false? [0.0 to 1.0]. T, I, and F are
independent. They need NOT sum to 1.0. Respond with ONLY:
\texttt{\{"T": <value>, "I": <value>, "F": <value>\}}.''

\subsection*{A.2\quad Strategy~2 (Probabilistic)}

\noindent\textbf{System:} ``You are a probabilistic classifier. You assign probabilities to three
mutually exclusive categories that MUST sum to exactly 1.0. Respond with ONLY a JSON object, no
other text.''

\medskip
\noindent\textbf{User:} ``Classify this statement into three mutually exclusive categories whose
probabilities sum to 1.0: Statement: \texttt{\{statement\}} --- T (True): Probability the statement
is true; I (Uncertain): Probability the truth value is unknown or undetermined; F (False): Probability
the statement is false. CONSTRAINT: T + I + F must equal 1.0. Respond with ONLY:
\texttt{\{"T": <value>, "I": <value>, "F": <value>\}}.''

\subsection*{A.3\quad Strategy~3 (Entropy-Derived)}

\noindent\textbf{System:} ``You are a binary truth estimator. You estimate the probability that a
statement is true (YES) versus false (NO). The two probabilities must sum to 1.0. Respond with ONLY
a JSON object, no other text.''

\medskip
\noindent\textbf{User:} ``Estimate the probability that this statement is true versus false:
Statement: \texttt{\{statement\}} --- P\_yes: Probability the statement is true, in the closed
interval [0.0, 1.0]; P\_no: Probability the statement is false, in the closed interval [0.0, 1.0].
CONSTRAINT: P\_yes + P\_no must equal 1.0. Respond with ONLY:
\texttt{\{"P\_yes": <value>, "P\_no": <value>\}}.''

\medskip
\noindent\textbf{Post-processing.} Indeterminacy is then derived externally from the Shannon binary
entropy of the elicited distribution:
\[
  I = -\bigl[p \cdot \log_2(p) + (1-p) \cdot \log_2(1-p)\bigr], \quad \text{where } p = P_{\text{yes}}.
\]
This yields a derived triple $(T, I, F) = (P_{\text{yes}},\, I,\, P_{\text{no}})$ which can then be
compared against Strategies~1 and~2 within a single notational frame.

\bibliographystyle{plainnat}
\bibliography{references}

@article{brown2020,
  author    = {Tom B. Brown and Benjamin Mann and Nick Ryder and others},
  title     = {Language models are few-shot learners},
  journal   = {Advances in Neural Information Processing Systems},
  volume    = {33},
  pages     = {1877--1901},
  year      = {2020}
}

@article{shorinwa2024,
  author    = {Oladapo Shorinwa and Zhiting Mei and Justin Lidard and Allen Ren and Anirudha Majumdar},
  title     = {A survey on uncertainty quantification of large language models},
  journal   = {arXiv preprint arXiv:2412.05563},
  year      = {2024}
}

@article{yadkori2024,
  author    = {Yasin Abbasi Yadkori and Ilja Kuzborskij and David Stutz and others},
  title     = {Mitigating {LLM} hallucinations via conformal abstention},
  journal   = {arXiv preprint arXiv:2405.01563},
  year      = {2024}
}

@inproceedings{gal2016,
  author    = {Yarin Gal and Zoubin Ghahramani},
  title     = {Dropout as a {B}ayesian approximation: representing model uncertainty in deep learning},
  booktitle = {Proceedings of the 33rd International Conference on Machine Learning (ICML)},
  pages     = {1050--1059},
  year      = {2016}
}

@inproceedings{guo2017,
  author    = {Chuan Guo and Geoff Pleiss and Yu Sun and Kilian Q. Weinberger},
  title     = {On calibration of modern neural networks},
  booktitle = {Proceedings of the 34th International Conference on Machine Learning (ICML)},
  pages     = {1321--1330},
  year      = {2017}
}

@inproceedings{velickovic2022,
  author    = {Petar Veli{\v{c}}kovi{\'c}},
  title     = {Softmax is not enough (for sharp size generalisation)},
  booktitle = {International Conference on Learning Representations (ICLR)},
  year      = {2022}
}

@article{hullermeier2021,
  author    = {Eyke H{\"u}llermeier and Willem Waegeman},
  title     = {Aleatoric and epistemic uncertainty in machine learning: an introduction to concepts and methods},
  journal   = {Machine Learning},
  volume    = {110},
  number    = {3},
  pages     = {457--506},
  year      = {2021}
}

@article{valdenegro2022,
  author    = {Matias Valdenegro-Toro},
  title     = {A deeper look into aleatoric and epistemic uncertainty estimation},
  journal   = {arXiv preprint arXiv:2204.09308},
  year      = {2022}
}

@inproceedings{kuhn2023,
  author    = {Lorenz Kuhn and Yarin Gal and Sebastian Farquhar},
  title     = {Semantic uncertainty: linguistic invariances for uncertainty estimation in natural language generation},
  booktitle = {International Conference on Learning Representations (ICLR)},
  year      = {2023}
}

@inproceedings{manakul2023,
  author    = {Potsawee Manakul and Adian Liusie and Mark J.F. Gales},
  title     = {{SelfCheckGPT}: zero-resource black-box hallucination detection for generative {LLMs}},
  booktitle = {Proceedings of the 2023 Conference on Empirical Methods in Natural Language Processing (EMNLP)},
  year      = {2023}
}

@book{smarandache1998,
  author    = {Florentin Smarandache},
  title     = {A Unifying Field in Logics: Neutrosophy. {N}eutrosophic Probability, Set, and Logic},
  publisher = {American Research Press},
  address   = {Rehoboth, NM, USA},
  year      = {1998}
}

@article{mason2026,
  author    = {Tony Mason},
  title     = {From scalars to tensors: declared losses recover epistemic distinctions that neutrosophic scalars cannot express},
  journal   = {arXiv preprint arXiv:2604.09602},
  year      = {2026}
}

@article{smarandache2018,
  author    = {Florentin Smarandache},
  title     = {Plithogenic Set: An Extension of Crisp, Fuzzy, Intuitionistic Fuzzy, and Neutrosophic Sets---Revisited},
  journal   = {Neutrosophic Sets and Systems},
  volume    = {21},
  pages     = {153--166},
  year      = {2018}
}

@article{gabriel2020,
  author    = {Iason Gabriel},
  title     = {Artificial intelligence, values, and alignment},
  journal   = {Minds and Machines},
  volume    = {30},
  number    = {3},
  pages     = {411--437},
  year      = {2020}
}

@article{bender2021,
  author    = {Emily M. Bender and Timnit Gebru and Angelina McMillan-Major and Shmargaret Shmitchell},
  title     = {On the dangers of stochastic parrots: can language models be too big?},
  booktitle = {Proceedings of the 2021 ACM Conference on Fairness, Accountability, and Transparency (FAccT)},
  pages     = {610--623},
  year      = {2021}
}

@article{priest2006,
  author    = {Graham Priest},
  title     = {In Contradiction: A Study of the Transconsistent},
  publisher = {Oxford University Press},
  year      = {2006}
}

@article{belnap1977,
  author    = {Nuel D. Belnap},
  title     = {A useful four-valued logic},
  booktitle = {Modern Uses of Multiple-Valued Logic},
  editor    = {J. Michael Dunn and George Epstein},
  publisher = {Reidel},
  pages     = {8--37},
  year      = {1977}
}

@article{carnielli2007,
  author    = {Walter Carnielli and Jo{\~a}o Marcos and Sandra de Amo},
  title     = {Formal inconsistency and evolutionary databases},
  journal   = {Logic and Logical Philosophy},
  volume    = {8},
  pages     = {115--152},
  year      = {2007}
}

@article{wei2022,
  author    = {Jason Wei and Xuezhi Wang and Dale Schuurmans and others},
  title     = {Chain-of-thought prompting elicits reasoning in large language models},
  journal   = {Advances in Neural Information Processing Systems},
  volume    = {35},
  year      = {2022}
}

@article{kojima2022,
  author    = {Takeshi Kojima and Shixiang Shane Gu and Machel Reid and Yutaka Matsuo and Yusuke Iwasawa},
  title     = {Large language models are zero-shot reasoners},
  journal   = {Advances in Neural Information Processing Systems},
  volume    = {35},
  year      = {2022}
}

@article{kong2023,
  author    = {Aobo Kong and Shiwan Zhao and Hao Chen and others},
  title     = {Better zero-shot reasoning with role-play prompting},
  journal   = {arXiv preprint arXiv:2308.07702},
  year      = {2023}
}

@article{kadavath2022,
  author    = {Saurav Kadavath and Tom Conerly and Amanda Askell and others},
  title     = {Language models (mostly) know what they know},
  journal   = {arXiv preprint arXiv:2207.05221},
  year      = {2022}
}

@article{bai2022,
  author    = {Yuntao Bai and Saurav Kadavath and Sandipan Kundu and others},
  title     = {Constitutional {AI}: harmlessness from {AI} feedback},
  journal   = {arXiv preprint arXiv:2212.08073},
  year      = {2022}
}

@article{shannon1948,
  author    = {Claude E. Shannon},
  title     = {A mathematical theory of communication},
  journal   = {Bell System Technical Journal},
  volume    = {27},
  number    = {3},
  pages     = {379--423},
  year      = {1948}
}

\end{document}